\documentclass{article}

\usepackage{microtype}
\usepackage{graphicx}
\usepackage{subfigure}
\usepackage{booktabs} %
\usepackage{hyperref}

\usepackage[accepted]{icml2020}

\usepackage{amsmath,amsfonts,bm}
\usepackage{amsthm}

\def\eqref#1{equation~\ref{#1}}

\def\1{\bm{1}}

\DeclareMathAlphabet{\mathsfit}{\encodingdefault}{\sfdefault}{m}{sl}
\SetMathAlphabet{\mathsfit}{bold}{\encodingdefault}{\sfdefault}{bx}{n}

\def\gU{{\mathcal{U}}}

\def\gX{{\mathcal{X}}}

\def\sS{{\mathbb{S}}}

\newcommand{\E}{\mathbb{E}}

\newcommand{\R}{\mathbb{R}}

\DeclareMathOperator*{\argmax}{arg\,max}

\newtheorem{theorem}{Theorem}
\newtheorem{definition}{Definition}

\newtheorem{corollary}{Corollary}

\newcommand{\indep}{\perp \!\!\! \perp}
\newcommand*\diff{\mathop{}\!\mathrm{d}}

\newcommand{\J}{\mathrm{\mathbf{J}}}

\usepackage{physics}
\usepackage{enumitem}
\usepackage{multirow}
\usepackage{xspace}
\usepackage{mathtools}

\newlength{\leftstackrelawd}
\newlength{\leftstackrelbwd}
\def\leftstackrel#1#2{\settowidth{\leftstackrelawd}%
{${{}^{#1}}$}\settowidth{\leftstackrelbwd}{$#2$}%
\addtolength{\leftstackrelawd}{-\leftstackrelbwd}%
\leavevmode\ifthenelse{\lengthtest{\leftstackrelawd>0pt}}%
{\kern-.5\leftstackrelawd}{}\mathrel{\mathop{#2}\limits^{#1}}}

\DeclareMathOperator{\SO}{SO}
\newcommand{\eg}{\textit{e.g.}}
\newcommand{\ie}{\textit{i.e.}}

\newcommand{\D}{\ensuremath{d}}
\newcommand{\cdf}{\textit{cdf}\xspace}

\newcommand{\citeposs}[1]{\citeauthor{#1}'s \citeyearpar{#1}}

\icmltitlerunning{The Power Spherical distribution}

\begin{document}

\twocolumn[
\icmltitle{The Power Spherical distribution}

\begin{icmlauthorlist}
\icmlauthor{Nicola De Cao}{uva,uoe}
\icmlauthor{Wilker Aziz}{uva}
\end{icmlauthorlist}

\icmlaffiliation{uva}{University of Amsterdam}
\icmlaffiliation{uoe}{The University of Edinburgh}

\icmlcorrespondingauthor{Nicola De Cao}{nicola.decao@uva.nl}

\icmlkeywords{Machine Learning, ICML}

\vskip 0.3in
]

\printAffiliationsAndNotice{}  %

\begin{abstract}
There is a growing interest in probabilistic models defined in hyper-spherical spaces, be it to accommodate observed data or latent structure. The von Mises-Fisher (vMF) distribution, often regarded as the Normal distribution on the hyper-sphere, is a standard modeling choice: it is an exponential family and thus enjoys important statistical results, for example, known Kullback-Leibler (KL) divergence from other vMF distributions. Sampling from a vMF distribution, however, requires a rejection sampling procedure which besides being slow poses difficulties in the context of stochastic backpropagation via the \textit{reparameterization trick}. Moreover, this procedure is numerically unstable for certain vMFs, \eg, those with high concentration and/or in high dimensions. We propose a novel distribution, the \textbf{Power Spherical} distribution, which retains some of the important aspects of the vMF (\eg, support on the hyper-sphere, symmetry about its mean direction parameter, known KL from other vMF distributions) while addressing its main drawbacks (\ie, scalability and numerical stability). We demonstrate the stability of Power Spherical distributions with a numerical experiment and further apply it to a variational auto-encoder trained on MNIST. Code at: \href{https://github.com/nicola-decao/power_spherical}{github.com/nicola-decao/power\_spherical}
\end{abstract}

\section{Introduction} \label{sec:introduction}

Manifold learning and machine learning applications of directional statistics~\citep{sra2018directional} have spurred interest in distributions defined in non-Euclidean spaces (\eg, simplex, hyper-torus, hyper-sphere). 
Examples include learning rotations~\citep[\ie, $ \SO(n) $,][]{falorsi2018explorations, falorsi2019reparameterizing}
and hierarchical structures on hyperbolic spaces~\citep{mathieu2019continuous,nagano2019wrapped}.
Hyper-spherical distributions, in particular, find applications in clustering~\citep{banerjee2005clustering,bijral2007mixture}, mixed-membership models~\citep{reisinger2010spherical}, computer vision~\citep{liu2017sphereface}, and natural language processing~\citep{kumar2018mises}. 

The von Mises-Fisher distribution~\citep[vMF;][]{mardia2009directional} is a natural and standard choice for densities in hyper-spheres. It is a two-parameter exponential family, one parameter being a mean direction and the other a scalar concentration, and it is symmetric about the former.
Because of that, it is often regarded as the Normal distribution on spheres.
Amongst other useful properties, it has closed-form Kullback–Leibler (KL) divergence with other vMF densities including the uniform distribution, one of its special cases.

\begin{figure}[t]
\centering
\subfigure[On the circle $\sS^1$.]{\includegraphics[width=0.45\linewidth]{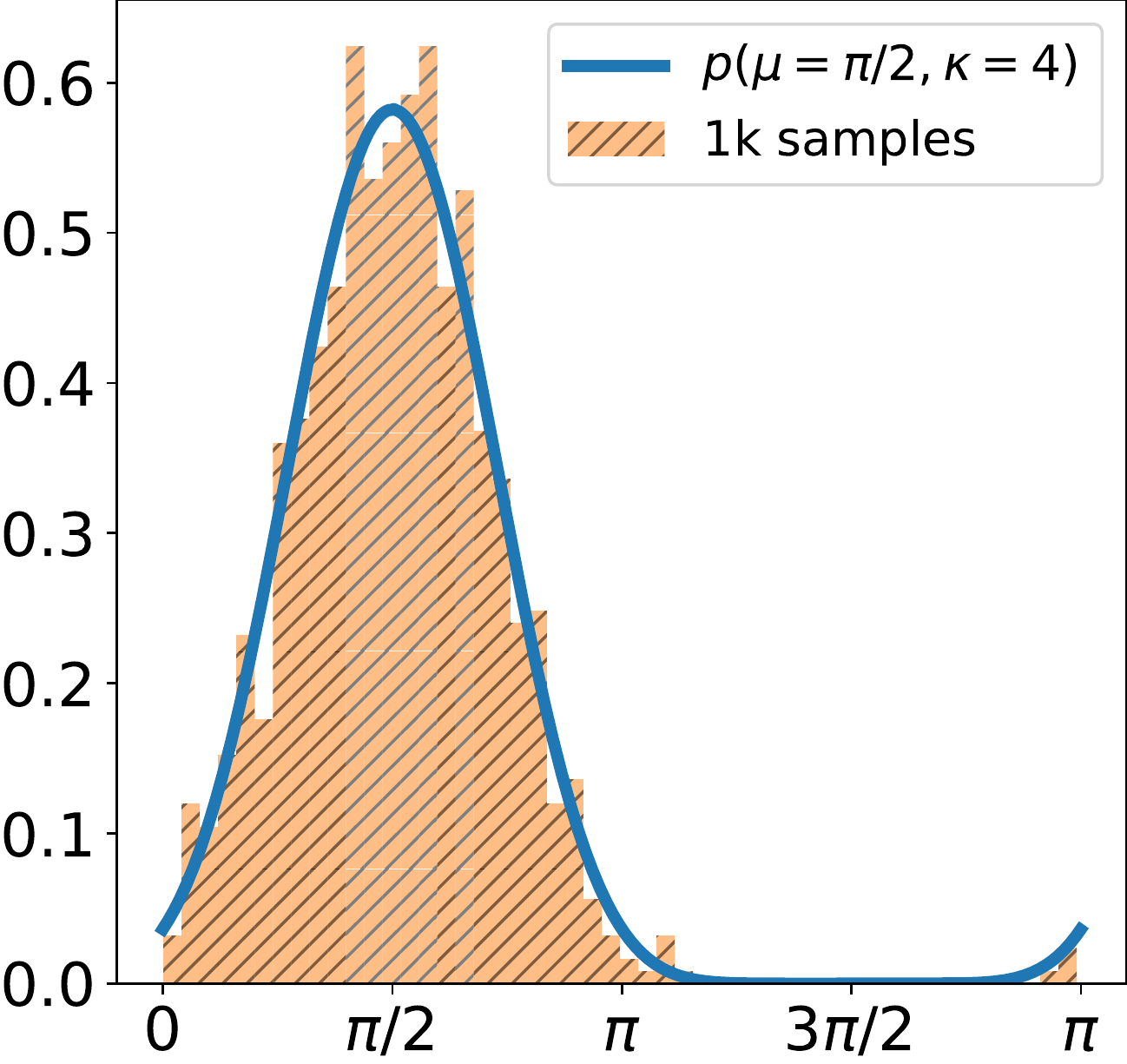}}
~
\subfigure[On the sphere $\sS^2$.]{\includegraphics[width=0.45\linewidth]{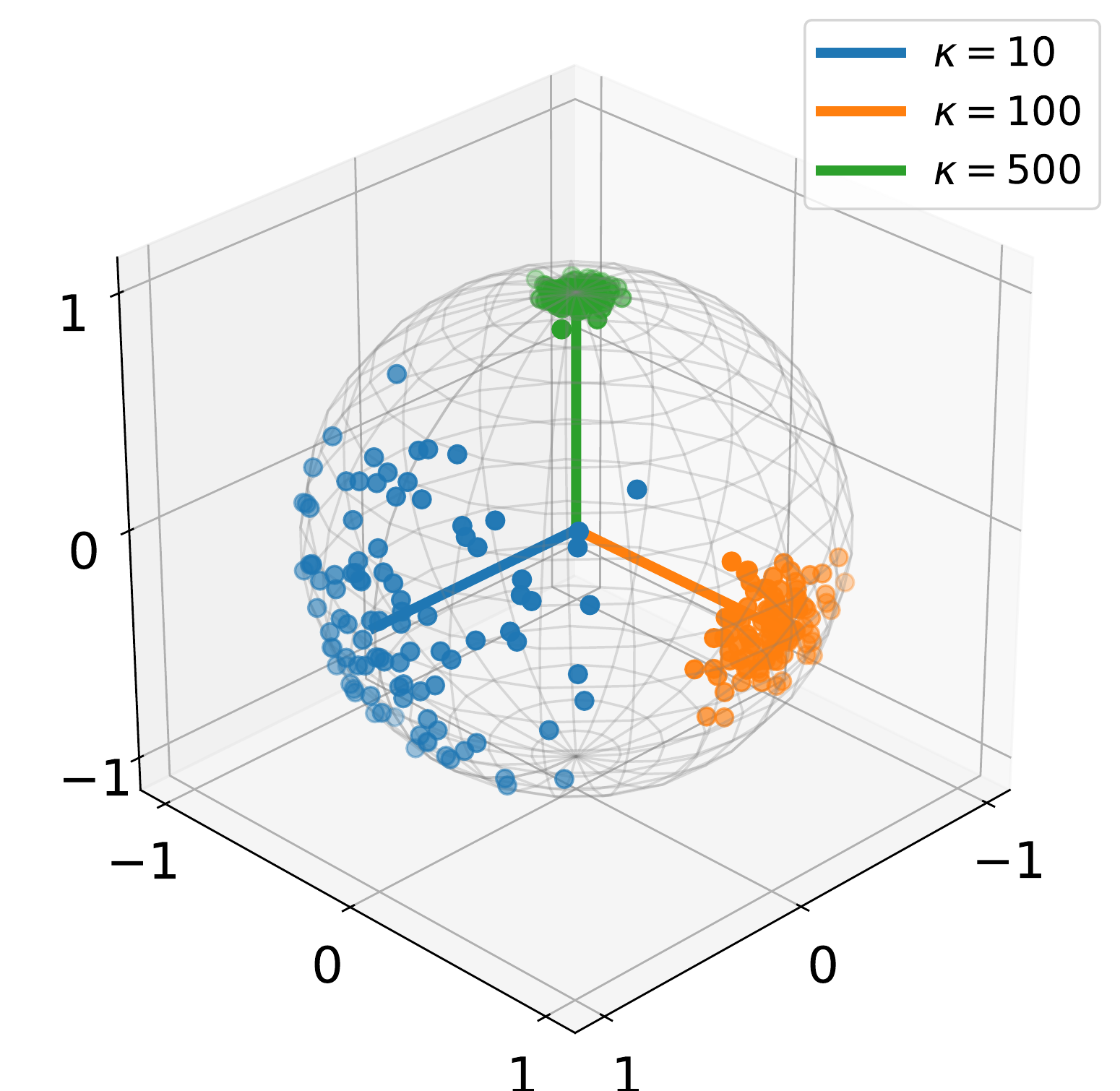}}
\caption{Example of draws and density function of the Power Spherical distribution. For the circle (a) we plot both the density and histograms for $1$k samples. For the sphere (b) we plot draws from $3$ distributions of orthogonal directions ($\mu$) and different concentration parameter $\kappa$.}
\label{fig:2d_3d_plots}
\end{figure}

Thanks to the \textit{tangent-normal decomposition}, sampling from a vMF, no matter its dimensionality, requires only sampling from a univariate marginal distribution. Unfortunately, the inverse cumulative density function (\cdf) of this marginal is not known analytically, which prevents straight-forward generation of independent samples. Fortunately, a rejection sampling procedure is known~\citep{ulrich1984computer}, but unsurprisingly, rejection sampling is inefficient. 

In deep learning, an appealing use of the vMF distribution is as a random generator in stochastic and differentiable computation graphs. 
For that, we need a \textit{reparameterization trick} to enable unbiased and low variance estimates of gradients of samples with respect to the vMF parameters~\citep{rezende2014stochastic, kingma2013auto}. With rejection sampling, reparameterization does not come naturally, requiring a correction term which has high variance~\citep{naesseth2017reparameterization}. 
This plays against widespread use of vMFs. 
For example, \citet{davidson2018hyperspherical} successfully used the vMF distribution to approximate the posterior distribution of a hyper-spherical variational auto-encoder~\citep[VAE;][]{kingma2013auto}, but they had to omit the correction term, trading variance for bias.
Additionally, due to its exponential form as well as a dependency on the modified Bessel function of the first kind~\citep{weisstein2002bessel}, the vMF distribution is numerically unstable in high dimensions or with high concentration~\citep{davidson2018hyperspherical}.

To overcome all of the vMF drawbacks,  we propose the novel \textbf{Power Spherical distribution}. We start from the \textit{tangent-normal decomposition} of vectors in hyper-spheres and specifically design a univariate marginal distribution that admits an analytical inverse \cdf. This marginal is used to derive a distribution on hyper-spheres of any dimension.
The resulting distribution is not an exponential family and is defined via a power law. Crucially, it is numerically stable and dispenses with rejection sampling for independent sampling.
We verify the stability of the Power Spherical distribution with a numerical experiment as well as reproducing some of the experiments of~\citet{davidson2018hyperspherical} while substituting the vMF with our proposed distribution.

\paragraph{Contributions} We propose a new distribution defined on any $d$-dimensional hyper-sphere which
\begin{itemize}[topsep=0pt,itemsep=0pt]
    \item has closed form marginal \cdf and inverse \cdf, thus it does not require rejection sampling;
    \item is fully reparameterizable without a correction term;
    \item is numerically stable in high dimensions and/or high concentrations. %
\end{itemize}

\section{Method}
We start with an overview of the vMF distribution, also introducing results that help formulate the Power Spherical distribution.
We then define the Power Spherical and present some of its proprieties such as mean, mode, variance, and entropy as well as its Kullback–Leibler divergence from a vMF and from a uniform distribution.

\subsection{Preliminaries}
Let $\sS^{\D-1} = \{x \in \R^\D: \|x\|_2 = 1\}$ be the hyper-spherical set.
A key idea in directional distribution theory is the \textit{tangent-normal decomposition}.

\begin{theorem}[9.1.2 in \citet{mardia2009directional}] \label{thm:tangent}
Any unit vector $x \in \sS^{\D-1}$ can be decomposed as 
\begin{equation}
    x = \mu t  + v (1 - t^2)^\frac12 \;,
\end{equation}
with $t \in [-1,1]$ and $v \in \sS^{\D-2}$ a tangent to $\sS^{\D-1}$ at $\mu$.
\end{theorem}

\begin{corollary}[9.3.1 in \citet{mardia2009directional}] \label{cor:distribution_dot}
The intersection of $\sS^{\D-1}$ with the plane through $t\mu$ and normal to $\mu$ is a $(\D-2)$-sphere of radius $\sqrt{1 - t^2}$. Moreover, $t$ has density
\begin{equation}
p_T(t;\D) \propto \left(1-t^2\right)^{\frac{\D-3}{2}} \quad \text{with} \quad t \in [-1,1] \;.
\end{equation}
\end{corollary}

Importantly, it follows from Theorem~\ref{thm:tangent} and Corollary~\ref{cor:distribution_dot} that every distribution that depends on $x$ only through $t = \mu^\top x$
can be expressed in terms of a \textbf{marginal distribution} $p_T$ and a \textbf{uniform distribution} $p_{V}$ on the subspace $\sS^{\D -1}$.
Density evaluation as well as sampling \textit{just} requires dealing with the marginal $p_T$ as $p_V$ has constant density and it is trivial to sample from.
An example from this class is the von Mises-Fisher distribution.

\begin{definition}
Let's define a 
unnormalized density 
as
\begin{equation}
p_X(x;\mu,\kappa) \propto \exp\left( \kappa \mu^\top x \right) \quad \text{with} \quad x \in \sS^{d-1} \;,
\end{equation}
with direction $\mu \in \sS^{\D-1}$ and concentration $\kappa \in \R_{\geq 0}$. When normalized, this is the von Mises-Fisher distribution.
\end{definition}

\begin{theorem}[9.3.12 in \citet{mardia2009directional}] \label{thm:marginal_vmf}
The marginal $p_T(t;\D,\kappa)$ of a von Mises-Fisher distribution is
\begin{equation}
p_T(t;\D,\kappa) = C_T(\kappa,\D) \cdot e^{\kappa t}\left(1-t^{2}\right)^{\frac{\D - 3}{2}} \;,
\end{equation}
with normalizer $C_T(\kappa,\D) = $
\begin{equation}
    \left(\frac{\kappa}{2}\right)^{\frac{\D}{2}-1}\left\{\Gamma\left(\frac{\D-1}{2}\right) \Gamma\left(\frac{1}{2}\right) I_{\frac{\D-1}{2}}(\kappa)\right\}^{-1} \;,
\end{equation}
where $I_{v}(z)$ is the modified Bessel function of the first kind.
\end{theorem}

Although evaluation of $p_T$ is tractable,
the Bessel function~\cite{weisstein2002bessel} (see Appendix~\ref{app:definitions}) is slow to compute and unstable for large arguments. Besides, $p_T$ does not admit a known closed-form \cdf (nor its inverse). Thus, rejection sampling is normally used to draw samples from it~\citep{ulrich1984computer}.

\subsection{The Power Spherical distribution}
As all the issues of the vMF distribution stem from a problematic marginal, we address them by defining a new distribution that shares some basic proprieties of a vMF but none of its drawbacks. Namely \emph{i)} it is rotationally symmetric about $\mu$, \emph{ii)} it can be expressed in terms of a marginal distribution $p_T$, and \emph{iii)} this marginal has closed-form and stable \cdf (and inverse \cdf).

\begin{definition} Let's define an unnormalized density as
\begin{equation}
p_X(x;\mu,\kappa) \propto \left(1 + \mu^\top x \right)^ \kappa \quad \text{with} \quad x \in \sS^{d-1} \;,
\end{equation}
with direction $\mu \in \sS^{\D-1}$ and concentration $\kappa \in \R_{\geq 0}$. When normalized this is the \textbf{Power Spherical} distribution.
\end{definition}

As we show in Theorem~\ref{thm:marginal_ps} (Appendix~\ref{app:marginal}), the marginal of the Power Spherical distribution has the valuable property of being defined in terms of an affine transformation of a Beta-distributed variable, \ie,
\begin{equation}
T = 2 Z - 1 \quad \text{with}\quad Z \sim\operatorname{Beta}\left(\alpha, \beta \right) \;,
\end{equation}
where $\alpha = \frac{\D-1}{2} + \kappa$ and $\beta = \frac{\D-1}{2}$. Therefore, its density can be easily assessed via the change of variable theorem (Theorem~\ref{thm:change_density} in Appendix~\ref{app:theorems}). Sampling and evaluating a Beta distribution is numerically stable and, crucially, it permits backpropagation though sampling via implicit reparameterization gradients~\citep{figurnov2018implicit}. The properly normalized density of the Power Spherical distribution is derived in Theorem~\ref{thm:power_spherical_density} (Appendix~\ref{app:power_spherical_full}) and is $p_X(x;\mu,\kappa)=$
\begin{equation}
    \underbrace{\left\{ 2^{\alpha + \beta} \pi^{\beta} \frac{\Gamma\left(\alpha\right)}{\Gamma\left(\alpha + \beta\right)} \right\}^{-1}  }_{=N_X(\kappa,\D)\ \text{(normalizer)}} \left( 1 + \mu^\top x \right)^\kappa \;.
\end{equation}

\begin{algorithm}[t]
\caption{Power Spherical sampling}
\label{alg:sampling}
\begin{algorithmic}
    \STATE {\bfseries Input:} dimension $p$, direction $\mu$, concentration $\kappa$
    \STATE sample $z \sim \operatorname{Beta}\left(Z;(d-1) / 2 + \kappa, (d-1) / 2 \right)$
    \STATE sample $v \sim \gU(\mathcal{S}^{d-2})$
    \STATE $t \gets 2 z - 1$
    \STATE $y \gets [t; (\sqrt{1-t^2}) v^\top \ ]^\top$ \COMMENT{concatenation}
    \STATE $\hat u \gets e_1 - \mu$ \COMMENT{$e_1$ is the base vector $[1,0,\cdots,0]^\top$}
    \STATE $u = \frac{\hat u}{\| \hat u \|_2}$
    \STATE $x \gets (I_d - 2 u u^\top) y$ \COMMENT{$I_d$ is the identity matrix $d \times d$}
    \STATE{\bfseries Return:}{ $x$}
\end{algorithmic}
\end{algorithm}

\paragraph{Sampling}
As for the vMF,\footnote{This is equal to the method of~\citet{davidson2018hyperspherical} (Algorithms~1 and~3) for sampling from a vMF where we use the marginal of the Power Spherical instead.} draws are obtained sampling
\begin{equation}
    t\sim p_T(t;\kappa, \D) \quad \text{and} \quad v\sim \gU(\sS^{\D-2}) \;,
\end{equation}
and constructing $y = [t;  v^\top \sqrt{1 - t^2}]^\top$ using Theorem~\ref{thm:tangent}. Finally, we apply a Householder reflection about $\mu$ to $y$ to obtain a sample $x$ (see Algorithm~\ref{alg:sampling}). All these operations are differentiable which means we can use the reparameterization trick to have low variance and unbiased estimation of gradients of Monte Carlo samples with respect to the parameters of the density~\citep{rezende2014stochastic, kingma2013auto}. Importantly, and differently from a vMF, sampling from a Power Spherical does not require rejection sampling. This leads to two main advantages: \emph{i)} fast sampling (as we demonstrate in Section~\ref{sec:experiments}), and \emph{ii)} no need for a high variance gradient correction term that compensates for sampling from a proposal distribution rather than the true one~\citep{naesseth2017reparameterization, davidson2018hyperspherical}.

\begin{table}[t]
    \centering
    \small
    \begin{tabular}{l l}
    \toprule 
    \textbf{Property} & \textbf{Value} \\
    \midrule
    $\E[X]$ & $\mu (\alpha - \beta) / (\alpha + \beta)$ \\
    $\operatorname{var}(X)$ & $\frac{2\alpha}{(\alpha+\beta)^{2}(\alpha+\beta+1)} \left((\beta - \alpha)  \mu \mu^\top +  (\alpha + \beta) I_d \right) $ \\
    Mode & $\mu \qquad $ (for $\kappa > 0$) \\
    $\operatorname{H}(T)$ & $\operatorname{H}(\operatorname{Beta}(\alpha, \beta)) + \log 2$\\
    $\operatorname{H}(X)$ & $\log N_X(\kappa,\D) - \kappa \big( \log 2 + \psi\left(\alpha\right) - \psi\left(\alpha + \beta\right) \big)$\\
    \bottomrule
    \end{tabular}
    \vspace{-.5em}
    \caption{Properties of $X \sim \operatorname{Power Spherical}(\mu, \kappa)$. Recall that $\alpha = (d-1)/2 + \kappa$ and $\beta = (d-1)/2$.}
    \label{tab:properties}
\end{table}

\subsection{Proprieties}
Table~\ref{tab:properties} summarizes some basic properties of the Power Spherical distribution. See Appendix~\ref{app:properties} for derivations.
In particular, note that having a closed-form differential entropy allows using the Power Spherical in applications such as variational inference~\citep[VI;][]{Jordan+1999:VI} and mutual information minimization.

\paragraph{Kullback–Leibler divergence}
The KL divergence between a Power Spherical $P$ and a uniform $Q=\gU(\sS^{d-1})$ is
\begin{equation}
\begin{aligned}
    \mathrm{D_{KL}}[P \| Q] &= - \operatorname{H}(P) + \operatorname{H}(Q) \;.
\end{aligned}
\end{equation}
See Theorem~\ref{thm:kl_uniform} (Appendix~\ref{app:kl}) for the full derivation.
Being able to compute the KL divergence from a uniform distribution in closed-form is useful in the context of variational inference as $Q$ can be used as a prior. 
Another important result we present here is a closed-form KL divergence of a Power Spherical distribution $P$ with parameters $\mu_p, \kappa_p$ from a vMF $Q$ with parameters $\mu_q, \kappa_q$, which evaluates to
\begin{equation}
- \operatorname{H}(P) + \log C_X(\kappa_q,\D) - \kappa_q \mu_q^\top \mu_p \left( \frac{\alpha -  \beta}{\alpha + \beta} \right)
\end{equation}
with $C_X(\kappa_q,\D)$ the vMF normalizer (see Theorem~\ref{thm:kl_vmf} in Appendix~\ref{app:kl}). This is valuable when using the Power Spherical to approximate a vMF, as we can assess the quality of the approximation. For example, if one has a vMF prior, it is then straightforward to use a Power Spherical approximate posterior in VI. 
If a vMF is necessary for a particular application, we can train with Power Spherical (enabling fast sampling and stable optimization) and then return the vMF that is closest to it in terms of KL.

\section{Experiments} \label{sec:experiments}
In this section we aim to show that Power Spherical distributions are more stable than vMFs, they are also faster to sample from, and lead to comparable performance when used in the context of variational auto-encoders.

\paragraph{Stability}
We tested numerical stability of both distributions for dimensions $d \in \{a \cdot 10^b\}$ and concentrations $\kappa \in \{a \cdot 10^b\}$ for all $a \in \{1,\dots,9\}$, $b \in \{0,\dots,5\}$. For every combination of $\langle d, \kappa \rangle$, we sample $10$ vectors $x^{(i)}$ and compute the gradient $g^{(i)}=\nabla_\kappa \mu^\top x^{(i)}$. If at least one of the samples $x^{(i)}$ or one of the gradients $g^{(i)}$ returns \emph{Not a Number} (NaN), we mark $\langle d, \kappa \rangle$ as unstable for that distribution.
In Figure~\ref{fig:stability} we show the regions of instability. As intended, the Power Spherical does not present numerical issues in these intervals, while the vMF does. This makes our distribution more suitable where high dimensional vectors are needed such as for language modelling~\citet{kumar2018mises}.

\paragraph{Efficiency}
We also compare sampling efficiency between the Power Spherical and the vMF to highlight that rejection sampling is an undesirable bottleneck. We measured sampling time with $\D=64$ of a batch of $100$ vectors of various concentrations $\kappa \in \{a \cdot 10^b\}$ with $a \in \{1,\dots,5\}$, $b \in \{0,\dots,4\}$.\footnote{On a NVIDIA Titian X 12GB GPU.} For every concentration $\kappa$ we computed mean and variance of $7$ trials that consisted of computing the mean execution time (in milliseconds) sampling $100$ times. Figure~\ref{fig:runtime}  shows the results. Sampling from a Power Spherical is at least $6\times$ faster than sampling from a vMF. For some concentrations where the rejection ratio is worse, it is almost $20\times$ faster. Noticeably, and differently from a vMF, sampling time is constant regardless of the concentration $\kappa$.

\begin{figure}[t]
\centering
\subfigure[Stability of the vMF distribution. Ours does not have numerical issues in these intervals. \label{fig:stability}]{\includegraphics[width=0.48\linewidth]{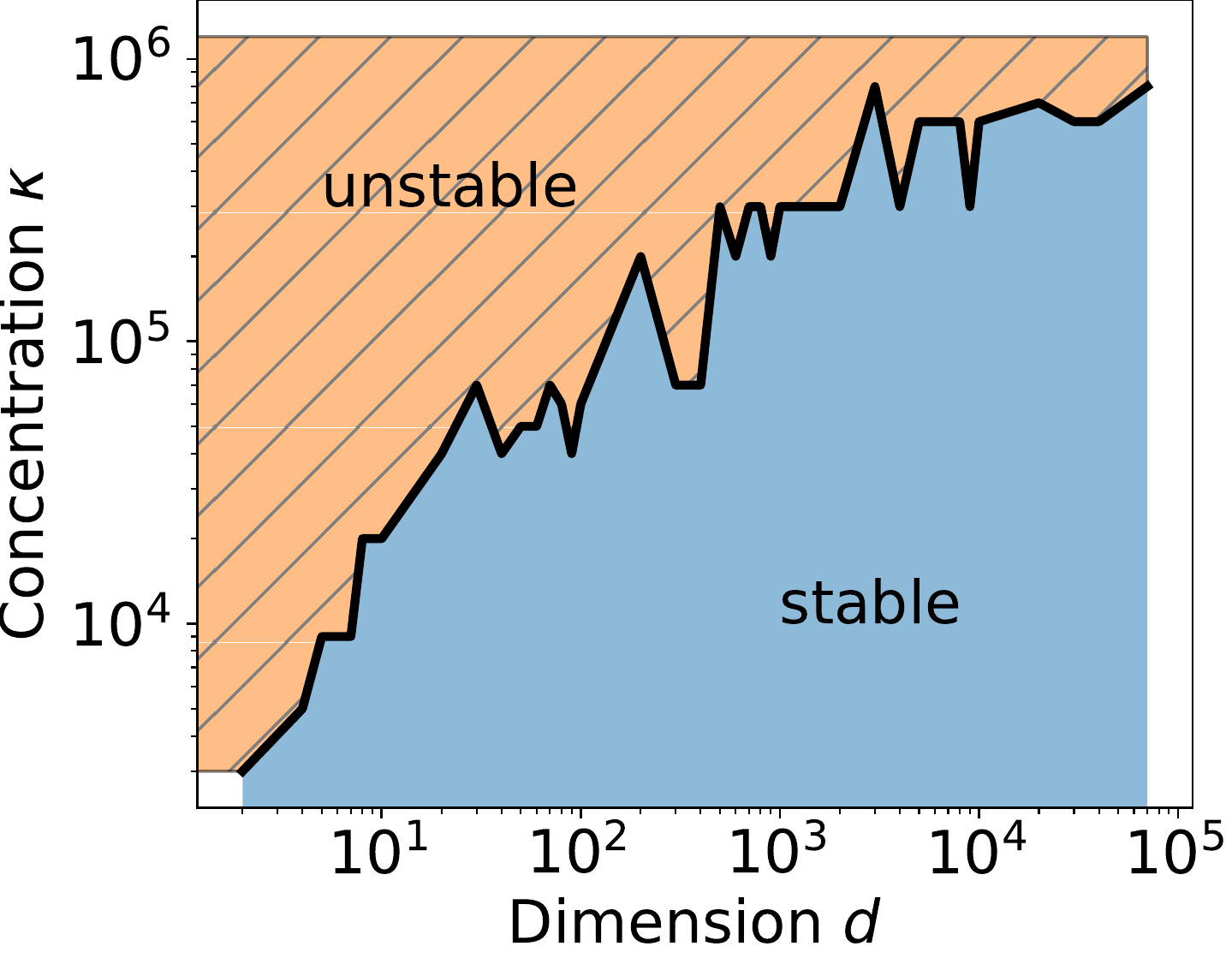}} 
~
\subfigure[Sampling time (on GPU) with $\D=64$ of a batch of $100$ vectors of varius concentrations. \label{fig:runtime}]{\includegraphics[width=0.48\linewidth]{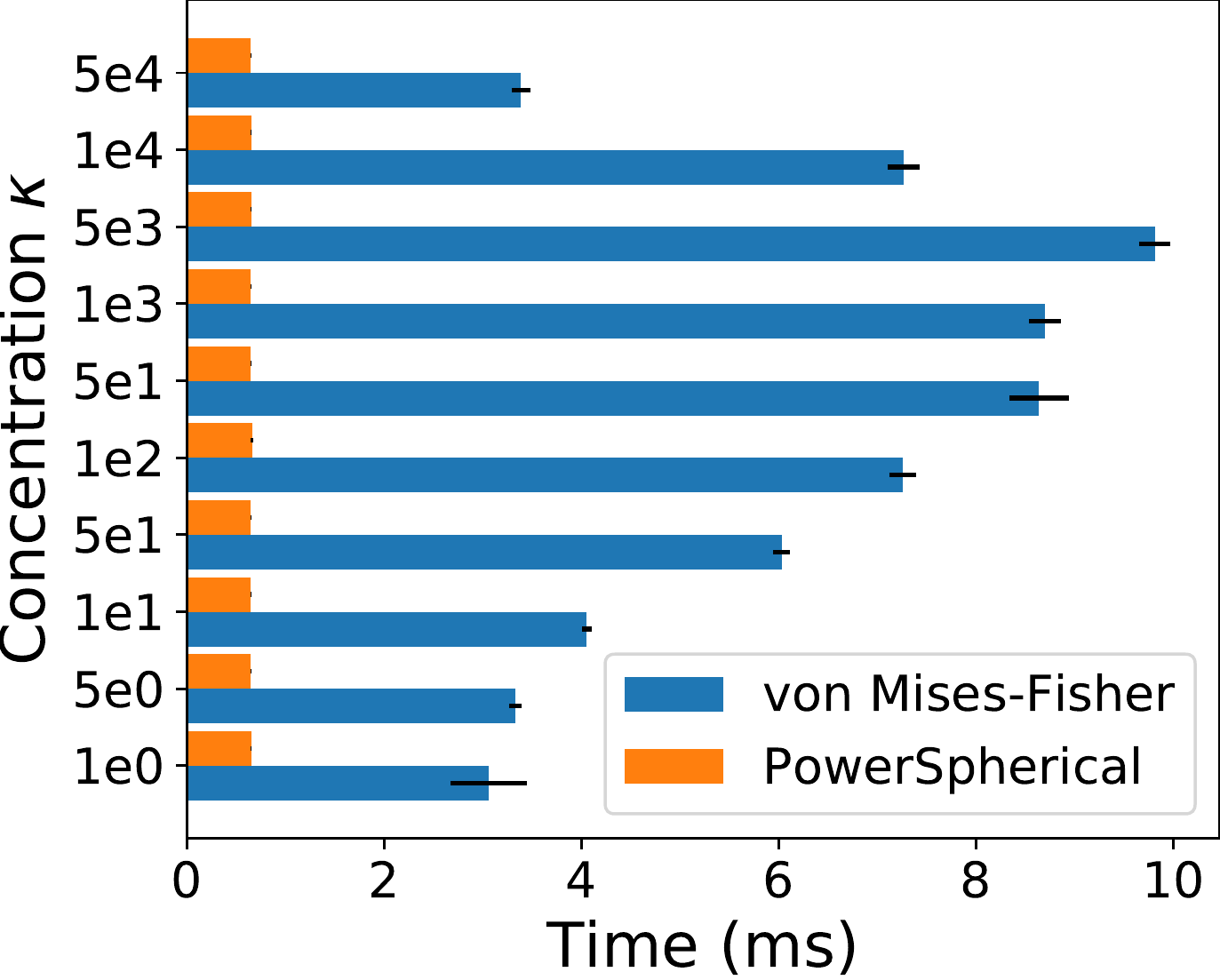}}
\vspace{-1.5em}
\caption{Comparing stability (a) and running time (b) of the von Mises-Fisher and the Power Spherical distribution.}
\label{fig:stability_runtime}
\end{figure}

\paragraph{Variational inference}

Finally, we employed our Power Spherical distribution in the context of variational auto-encoders~\citep{kingma2013auto}. In particular, we replicated some of~\citeposs{davidson2018hyperspherical} experiments comparing the Power Spherical and the vMF on the the MNIST dataset~\citep{lecun1998gradient}. We implement a simple feed-forward encoder $[784, 256, \tanh, 128, \tanh, d]$ and decoder $[d, 128, \tanh, 256, \tanh 784]$ and optimize the evidence lower bound for $100$ epochs for $d \in \{5,10,20,40\}$. We used Adam~\cite{kingma2014adam} with learning rate $10^{-3}$ and batch size $64$.
Table~\ref{tab:mnist} shows importance sampling estimates (with $5$k Monte Carlo samples) of log-likelihood and the evidence lower bound on the test set. We observe no substantial difference in performance between the two distributions. 
This shows, across dimensions, that the Power Spherical is sufficiently expressive to replace the vMF in this application. 
However, training with the Power Spherical was $>\!2 \times$ faster than with the vMF. One can notice that both log-likelihood and ELBO do not improve after $d=40$. This is in line with \citeposs{davidson2018hyperspherical} findings and a consequence of a shallow architecture. It is not the purpose of this experiment to show improvements on this task.

\section{Conclusion and future work}

We presented a novel distribution on the $d$-dimensional sphere that, unlike the typically used von Mises-Fisher, \textit{i)} is numerically stable in high dimensions and concentration, \textit{ii)} has gradients of its samples with respect to its parameters, and \textit{iii)} does not require rejection sampling allowing faster computation and exact reparameterization gradient without a high variance correction term. 
We empirically show that our distribution is more numerically stable and faster to sample from compared to a vMF while preforming equally well in a variational auto-encoder setting. As shown in our experiments, high dimensional hyperspaces suffer from surface area collapse at the expense of the expressivity of latent space embeddings. \citet{davidson2019increasing} addressed some of these issues, future work might explore this direction further. Another future direction is to use the Power Spherical in combination with normalizing flows~\citep{rezende2015variational}. Though neural autoregressive flows~\citep{huang2018neural,bnaf19} have been shown to be remarkably flexible, they are still subject to topological constraints~\citep{dinh2019rad,cornish2019localised}, which motivates extensions for complex manifolds~\citep{brehmer2020flows,wu2020stochastic} including spheres and tori~\citep{rezende2020normalizing}. In addition, a recent work by~\citet{falorsi2020neural} extend Neural Ordinary Differential Equations~\citep{chen2018neural} to arbitrary Riemannian manifolds allowing the use of continuous normalizing flow to such spaces.

\begin{table}[t]
\centering
\bgroup
\setlength\tabcolsep{5.5pt}
\begin{tabular}{l|cc|cc}
    \toprule
    \multirow{2}{*}{\textbf{Method}} &
    \multicolumn{2}{c}{\textbf{vMF}} &
    \multicolumn{2}{c}{\textbf{Power Spherical}} \\
    & LL & ELBO & LL & ELBO  \\
    \midrule
    $\D=5$ & $-114.51$ & $-117.68$ & $-114.49$ & $-118.01$ \\
    $\D=10$ & $-97.37$ & $-101.78$ & $-97.46$ & $-101.86$ \\
    $\D=20$ & $-93.80$ & $-99.38$ & $-93.70$ & $-99.27$ \\
    $\D=40$ & $-98.64$ & $-108.44$ & $-98.63$ & $-108.32$ \\
    \bottomrule
\end{tabular}
\egroup
\caption{Comparison between the vMf and Power Spherical distributions in a VAE on MNIST with different dimensional latent spaces $\sS^{d-1}$. We show estimated (with $5$k Monte Carlo samples) log-likelihood (LL) and evidence lower bond (ELBO) on test set.}
\label{tab:mnist}
\end{table}

\clearpage

\subsection*{Acknowledgements}
Authors want to thank Luca Falorsi and Sergio De Cao for helpful discussions and technical support as well as Srinivas Vasudevan and Robin Scheibler for spotting some typos.
This project is supported by SAP Innovation Center Network, ERC Starting Grant BroadSem (678254).

\bibliography{main}
\bibliographystyle{icml2020}

\clearpage
\appendix

\section{Definitions} \label{app:definitions}

\begin{definition}
The Gamma function is:
\begin{equation}
	\Gamma(x) = \int_0^\infty u^{x-1} \exp(-u) \diff u \;.
\end{equation}
\end{definition}

\begin{definition}
The Beta function:
\begin{equation}
	B(a,b) = \frac{ \Gamma(a)  \Gamma(b)}{ \Gamma(a + b)}  \;.
\end{equation}
\end{definition}

\begin{definition}
The incomplete Beta function is:
\begin{equation}
	B_x(a,b) = \int_0^x u^{a-1} (1 - u)^{b-1} \diff u \;.
\end{equation}
\end{definition}

\begin{definition}
The regularized incomplete Beta function is:
\begin{equation}
    I_x(a,b)  = \frac{B_x(a,b)}{B(a,b)}  \;.
\end{equation}
\end{definition}

\begin{definition} \label{def:area}
The surface area of the hyper-sphere $\sS^{d-1}$ is:
\begin{equation}
	A_{d-1} = \frac{2 \pi^{\frac{d}{2}}}{\Gamma\left(\frac{d}{2}\right)} \;.
\end{equation}
\end{definition}

\begin{definition} \label{def:uniform}
The uniform distribution on $\sS^{d-1}$ has constant density equal to the reciprocal of the surface area (Definition~\ref{def:area}):
\begin{equation}
	p_X(x) = \frac{1}{A_{d-1}} \;.
\end{equation}
\end{definition}

\begin{definition}
The modified Bessel function of the first kind is defined as
\begin{equation}
I_{v}(z)=\left(\frac{1}{2} z\right)^{v} \sum_{k=0}^{\infty} \frac{\left(\frac{1}{4} z^{2}\right)^{k}}{k ! \Gamma(v+k+1)} \;.
\end{equation}
\end{definition}

An useful integral function
\begin{equation} \label{eq:useful_integral}
\begin{aligned}
    &= \int (1+x)^a (1-x)^b \diff x  \\
    &= 2^{a + b + 1} B_{\frac{x + 1}{2}}\left(a + 1, b + 1\right) + C  \;.
\end{aligned}
\end{equation}

\begin{definition} \label{def:entropy}
For a distribution $P$ of a continuous random variable 
the differential entropy is defined to be the integral
\begin{equation}
\begin{aligned}
\operatorname{H}(P)
&= -\E_{p(x)}\left[\log p(x) \right] \; ,
\end{aligned}
\end{equation}
where $p$ denotes the probability density function of $P$.
\end{definition}

\begin{definition} \label{def:kl}
For distributions $P$ and $Q$ of a continuous random variable 
the Kullback–Leibler divergence is defined to be the integral
\begin{equation}
\begin{aligned}
    \mathrm{D_{KL}}(P \| Q)
    &= - \operatorname{H}(P) - \E_{p(x)}\left[\log q(x) \right] \; ,
\end{aligned}
\end{equation}
where $p$ and $q$ denote the probability density function of $P$ and $Q$ respectively.
\end{definition}

\begin{definition} \label{def:beta_distribution}
A random variable $X$ is distributed according to the Beta distribution if its probability density function is 
\begin{equation}
B(\alpha, \beta)^{-1} x^{\alpha-1}(1-x)^{\beta-1} \;,
\end{equation}
for $x \in [0,1]$ and zero elsewhere where $\alpha \in \R^{>0}$ and $\beta \in \R^{>0}$ are shape parameters.
\end{definition}

\section{Theorems} \label{app:theorems}

\begin{theorem} \label{thm:beta_entropy}
The entropy of a random variable $X$ Beta distributed is
\begin{equation} \label{eq:beta_entropy}
\begin{aligned}
\operatorname{H}(X) &= \log B(\alpha, \beta) + (\alpha+\beta-2) \psi(\alpha+\beta) \\
&\quad  - (\alpha-1) \psi(\alpha)-(\beta-1) \psi(\beta) 
\end{aligned}
\end{equation}
where $\psi(x)=\frac{\partial}{\partial x} \log (\Gamma(x))$.
\end{theorem}

\begin{theorem} \label{thm:beta_expectation}
The expectation of $X$ of a random variable $X$ Beta distributed is
\begin{equation} \label{eq:beta_expectation}
\E[X] =  \frac{\alpha}{\alpha+\beta} \;.
\end{equation}
\end{theorem}

\begin{theorem} \label{thm:beta_expectation_ln}
The expectation of $\log X$ of a random variable $X$ Beta distributed is
\begin{equation} \label{eq:beta_expectation_ln}
\E[\log X] =  \psi(\alpha) - \psi(\alpha+\beta) \;.
\end{equation}
\end{theorem}

\begin{theorem} \label{thm:beta_expectation_pow2}
The expectation of $X^2$ of a random variable $X$ Beta distributed is
\begin{equation} \label{eq:beta_expectation_pow2}
\E[X^2] = \frac{(\alpha+1) \alpha}{(\alpha+\beta+1)(\alpha+\beta)} \;.
\end{equation}
\end{theorem}

\begin{theorem} \label{thm:beta_var}
The variance of $X$ of a random variable $X$ Beta distributed is
\begin{equation} \label{eq:beta_var}
\operatorname{var}[X] =  \frac{\alpha \beta}{(\alpha+\beta)^{2}(\alpha+\beta+1)} \;.
\end{equation}
\end{theorem}

\begin{theorem} \label{thm:change_density}
Given a bijective function $f: \mathcal{X} \rightarrow \mathcal{Y}$ between two continuous random variables $X \in\mathcal{X}\subseteq \R^d $ and $Y \in \mathcal{Y} \subseteq \R^d$, a relation between the probability density functions $p_Y(y)$ and $p_X(x)$ is
\begin{equation} \label{eq:flow}
 p_Y(y) = p_X(x) \abs{ \det \J_{f(x)} }^{-1} \;,
\end{equation}
where $y = f(x)$, and $\abs{ \det \mathbf{J}_{f(x)} }$ is the absolute value of the determinant of the Jacobian of $f$ evaluated at $x$.
\end{theorem}

\begin{theorem} \label{thm:entropy}
Given a bijective function $f: \mathcal{X} \rightarrow \mathcal{Y}$ between two continuous random variables $X \in\mathcal{X}\subseteq \R^d $ and $Y \in \mathcal{Y} \subseteq \R^d$, a relation between the two respective entropies $\operatorname{H}(Y)$ and $\operatorname{H}(X)$ is
\begin{equation}
\operatorname{H}(Y) = \operatorname{H}(X) + \int_\gX p_X(x) \log \left| \det \J_{f(x)} \right| \diff x \;,
\end{equation}
where $y = f(x)$, and $\abs{ \det \mathbf{J}_{f(x)} }$ is the absolute value of the determinant of the Jacobian of $f$ evaluated at $x$.
\end{theorem}

\begin{theorem}[9.3.33 in~\citet{mardia2009directional}] \label{thm:mean}
Continuous distributions with rotational symmetry about a direction $\mu$ and with probability density functions of the form $p_X \propto g(\mu^\top x)$ have mean
\begin{equation}
	 \E[X] = \E[T] \mu \;.
\end{equation}
\end{theorem}

\begin{theorem}[9.3.34 in~\citet{mardia2009directional}] \label{thm:var}
Continuous distributions with rotational symmetry about a direction $\mu$ and with probability density functions of the form $p_X \propto g(\mu^\top x)$ have variance
\begin{equation}
	 \operatorname{var}[X] = \operatorname{var}[T] \mu \mu^\top + \frac{1 - \E[T^2]}{d-1} (I_d - \mu \mu^\top) \;,
\end{equation}
where $I_d$ is the $d \times d$ identity matrix.
\end{theorem}

\section{Derivations} \label{app:derivations}

\subsection{The Power Spherical marginal} \label{app:marginal} 

\begin{theorem} \label{thm:marginal_ps}
Let $\sS^{d-1} = \{x \in \R^d: \|x\| = 1\}$ be the hyper-spherical set. Let an unnormalized density be
\begin{equation}
p_X(x;\mu,\kappa) \propto \left(1 + \mu^\top x \right)^ \kappa \quad \text{with} \quad x \in \sS^{d-1} \;,
\end{equation}
with $x\in \sS^{d-1}$, direction $\mu \in \sS^{d-1}$, and concentration parameter $\kappa \in \R_{\geq0}$. Let $T$ bet a random variable  that denotes the dot-product $t=\mu^\top x$, then 
\begin{equation}
T = 2 Z - 1 \quad \text{with}\quad Z \sim\operatorname{Beta}\left(\alpha, \beta \right) \;,
\end{equation}
where $\alpha = \frac{\D-1}{2} + \kappa$ and $\beta = \frac{\D-1}{2}$.

\end{theorem}
\begin{proof}
Given Corollary~\ref{cor:distribution_dot}, the marginal distribution of the dot-product $t$ is $\propto( 1 + t )^\kappa (1-t^2)^{\frac{d-3}{2}}$ so its normalizer is
\begin{align}
    N_T(\kappa, d) &= \int_{-1}^{1} \left( 1 + t \right)^\kappa \left(1-t^2\right)^{\frac{d-3}{2}} \diff t \\
    &= \int_{-1}^{1} \left( 1 + t \right)^{\frac{d-3}{2} + \kappa} \left(1-t\right)^{\frac{d-3}{2}} \diff t  \\
    &\leftstackrel{(\ref{eq:useful_integral})}{=} 2^{d + \kappa - 2} \Bigg(
    B_1\left(\frac{d-1}{2} + \kappa, \frac{d-1}{2} \right) \nonumber \\
    &\quad - \underbrace{B_0\left(\frac{d-1}{2} + \kappa, \frac{d-1}{2} \right)}_{=0} \Bigg) \\
    &= 2^{d + \kappa - 2} B\left(\frac{d-1}{2} + \kappa, \frac{d-1}{2} \right) \;.
\end{align}

It follows that the \textbf{probability density function} of the dot-product marginal distribution is $p_T(t;\kappa, d)=$
\begin{align}
	&= N_T(\kappa, d)^{-1} \left( 1 + t \right)^\kappa \left(1-t^2\right)^{\frac{d-3}{2}} \\
	&=  N_T(\kappa, d)^{-1} \left( 1 + t \right)^{\frac{d-3}{2} + \kappa} \left(1-t\right)^{\frac{d-3}{2}} \\
	&\stackrel{*}{=}  N_T(\kappa, d)^{-1} \left( 2 z \right)^{\frac{d-2}{2} + \kappa - 1} \left(2- 2z \right)^{\frac{d-1}{2} - 1} \\
	&= B\left(\alpha, \beta \right)^{-1} z^{\alpha - 1} \left(1- z \right)^{\beta - 1} \;. \label{eq:beta_distribution}
\end{align}
where $*$ indicates a substitution $t = 2z - 1$ and eventually $\alpha = \frac{d-1}{2} + \kappa$, $\beta = \frac{d-1}{2}$. Notice that Equation~\ref{eq:beta_distribution} is a Beta distribution. Therefore, if we define the random variable $Z\sim \operatorname{Beta}(\alpha, \beta)$ turns out that $T$ is simply $T=2Z -1$. 
\end{proof}
\begin{corollary}
Following Theorem~\ref{thm:marginal_ps}, the marginal of a Power Spherical distribution has \textbf{cumulative density function} $F(t;\kappa, d)=$
\begin{align}
&= N_T(\kappa, d)^{-1} \int_{-1}^t \left( 1 + x \right)^\kappa \left(1-x^2\right)^{\frac{d-3}{2}} \diff x \\
&= B\left(\alpha, \beta \right)^{-1} B_{\frac{x + 1}{2}}\left(\alpha, \beta \right) \\
&= I_\frac{t + 1}{2}\left( \alpha, \beta \right) \;,
\end{align}
and \textbf{inverse cumulative density function} 
\begin{equation}
    F^{-1}(y;\kappa, d) = 2 I^{-1}_y\left( \alpha, \beta \right) - 1 \;.
\end{equation}
\end{corollary}

\begin{corollary} \label{cor:entropy_marginal}
Using Theorem~\ref{thm:beta_entropy} and~\ref{thm:entropy} we derive the differential entropy of the marginal Power Spherical $p_T$. Using $t=2z-1 = f(z)$, $\det \J_{f(z)} = 2$ for all $z$, and then
\begin{align}
    \operatorname{H}(T) &= \operatorname{H}(Z) + \log 2 \\
    &\leftstackrel{(\ref{eq:beta_entropy})}{=} \log B(\alpha, \beta) +(\alpha+\beta-2) \psi(\alpha+\beta) \nonumber \\
    &\quad  - (\alpha-1) \psi(\alpha)-(\beta-1) \psi(\beta) + \log 2 \;.
\end{align}
\end{corollary}

\subsection{The normalized Power Spherical distribution} \label{app:power_spherical_full}
\begin{theorem} \label{thm:power_spherical_density}
The normalized density of the Power Spherical distribution is $p_X(x;\mu,\kappa)=$
\begin{equation}
	 \left\{ 2^{\alpha + \beta} \pi^{\beta} \frac{\Gamma\left(\alpha\right)}{\Gamma\left(\alpha + \beta\right)} \right\}^{-1} \left( 1 + \mu^\top x \right)^\kappa \;,
\end{equation}
with $\alpha = \frac{\D-1}{2} + \kappa$ and $\beta = \frac{\D-1}{2}$.

\begin{proof}
The Power Spherical is expressed via the tangent normal decomposition (Theorem~\ref{thm:tangent}) as a joint distribution between $T \sim p_T(t;\kappa, d)$ (from Theorem~\ref{thm:marginal_ps}) and $V \sim \gU(\sS^{d-2})$. Since $T \indep V$, the Power Spherical normalizer $N_X(p, \kappa)$ is the product of the normalizer of $p_T(t;\kappa, d)$ and the uniform distribution on $\sS^{d-2}$ (whose probability is constant on the $d-2$-sphere -- see Definition~\ref{def:uniform}), that is 
\begin{align}
N_X(\kappa, d) &= N_T(\kappa, d) \cdot A_{d - 2} \\
&= 2^{\alpha + \beta - 1} B\left(\alpha, \beta \right) \frac{2 \pi^{\beta}}{\Gamma\left(\beta\right)} \\
&= 2^{\alpha + \beta} \pi^{\beta} \frac{\Gamma\left(\alpha\right)}{\Gamma\left(\alpha + \beta\right)} \;.
\end{align}
Thus, $p_X(x;\mu,\kappa)=N_X(\kappa, d)^{-1} (1 + \mu^\top x)^\kappa$.
\end{proof}
\end{theorem}

\subsection{Power Spherical properties} \label{app:properties}
\begin{corollary} \label{cor:mean_ps}
Directly applying Theorem~\ref{thm:mean}, the mean of a Power Spherical is
\begin{equation}
    \E[X] = \E[T] \mu = (2 \E[Z] - 1) \mu \stackrel{(\ref{eq:beta_expectation})}{=} \left( \frac{\alpha - \beta}{\alpha + \beta} \right) \mu \;,
\end{equation}
with $\alpha = \frac{d-1}{2} + \kappa$, $\beta = \frac{d-1}{2}$.
\end{corollary} 

\begin{corollary} \label{cor:var_ps}
Directly applying Theorem~\ref{thm:var}, the variance of a Power Spherical is $\operatorname{var}[X] =$
\begin{align}
    &= \operatorname{var}[T] \mu \mu^\top + \frac{1 - \E[T^2]}{d-1} (I_d - \mu \mu^\top) \\
    &= 4 \operatorname{var}[Z] \mu \mu^\top + 4\frac{\E[Z] - \E[Z^2]}{2 \beta} (I_d - \mu \mu^\top) \\
    &\leftstackrel{(\ref{eq:beta_expectation_pow2}, \ref{eq:beta_var})}{=} 4 \operatorname{var}[Z] \mu \mu^\top + 4\frac{\operatorname{var}[Z](\alpha + \beta)}{2 \beta} (I_d - \mu \mu^\top) \\
    &= \frac{2\operatorname{var}[Z]}{\beta} \left((\beta - \alpha)  \mu \mu^\top +  (\alpha + \beta) I_d \right) \\
    &\leftstackrel{(\ref{eq:beta_var})}{=} \frac{2\alpha \left((\beta - \alpha)  \mu \mu^\top +  (\alpha + \beta) I_d \right) }{(\alpha+\beta)^{2}(\alpha+\beta+1)}  \;,
\end{align}
with $\alpha = \frac{d-1}{2} + \kappa$, $\beta = \frac{d-1}{2}$ and $I_d$ is the $d \times d$ identity matrix.
\end{corollary} 

\begin{theorem} \label{thm:mode_ps}
The mode of a Power Spherical with $\kappa > 0$ is
\begin{equation}
    \argmax_{x \in \sS^{d-1}} p_X(x;\mu, \kappa) = \mu \;.
\end{equation}
\begin{proof}
We have $\argmax\limits_{x \in \sS^{d-1}} p_X(x;\mu, \kappa)=$
\begin{align}
    &= \argmax_{x \in \sS^{d-1}} N_X(\kappa, d)^{-1} (1 + \mu^\top x)^\kappa \\
    &\overset{\kappa > 0}{=} \argmax_{x \in \sS^{d-1}} \log (1 + \mu^\top x) \\
    &= \argmax_{x \in \sS^{d-1}} \mu^\top x = \mu \;.
\end{align}
\end{proof}
\end{theorem}

\subsection{Power Spherical differential entropy} \label{app:entropy_full}

\begin{theorem} \label{thm:entropy_ps}
The differential entropy of the Power Spherical $p_X$ that is $\operatorname{H}(X)=$
\begin{equation}
    \log N_X(\kappa,d) - \kappa \left( \log 2 + \psi\left(\alpha\right) - \psi\left(\alpha + \beta\right) \right) \;,
\end{equation}
with $\alpha = \frac{d-1}{2} + \kappa$, $\beta = \frac{d-1}{2}$.
\begin{proof}
Applying Definition~\ref{def:entropy}, $\operatorname{H}(X)=$
\begin{align}
    &= - \E_X[ \log q(X)] \\
    &= \log N_X(\kappa, d) - \kappa\ \E_X[ \log (1 + \mu^\top X)] \\
    &= \log N_X(\kappa, d) - \kappa\ (\log 2 + \E_Z[ \log Z ]) \\
    &\leftstackrel{(\ref{eq:beta_expectation_ln})}{=} \log N_X(\kappa, d) - \kappa\ \big(\log 2 + \psi(\alpha)  - \psi(\alpha + \beta) \big) \;.
\end{align}
\end{proof}
\end{theorem}

\subsection{Kullback–Leibler divergence with the von Mises-Fisher distribution} \label{app:kl}

\begin{theorem} \label{thm:kl_vmf}
The Kullback–Leibler divergence $\mathrm{D_{KL}}$ (Definition~\ref{def:kl}) between a Power Spherical distribution $P$ with parameters $\mu_p, \kappa_p$ and von Mises-Fisher and $Q$ with parameters $\mu_q, \kappa_q$ is $\mathrm{D_{KL}}[P \| Q] = $
\begin{equation}
- \operatorname{H}(P) + \log C_X(\kappa_q, d) - \kappa_q \mu_q^\top \mu_p \left( \frac{\alpha - \beta}{\alpha + \beta} \right) \;,
\end{equation}
with $\alpha = \frac{d-1}{2} + \kappa$, $\beta = \frac{d-1}{2}$.
\begin{proof}
Applying Definition~\ref{def:kl}, $\mathrm{D_{KL}}[P \| Q] =$
\begin{align}
    &= - \operatorname{H}(P) - \E_{p_X}[ \log q(X)] \\
    &= - \operatorname{H}(P) + \log C_X(\kappa_q, d) - \E_{p_X}[\kappa_q \mu_q^\top X] \\
    &= - \operatorname{H}(P) + \log C_X(\kappa_q, d) - \kappa_q \mu_q^\top \E_{p_T}[T] \mu_p \\
    &= - \operatorname{H}(P) + \log C_X(\kappa_q, d) - \kappa_q \mu_q^\top \mu_p \left( \frac{\alpha - \beta}{\alpha + \beta} \right) \;,
\end{align}
where $C_X(\kappa_q, d)$ is the von Mises-Fisher normalizer.
\end{proof}
\end{theorem}

\subsection{Kullback–Leibler divergence with $\gU(\sS^{d-1})$} \label{app:kl_uniform}
\begin{theorem} \label{thm:kl_uniform}
The Kullback–Leibler divergence $\mathrm{D_{KL}}$ (Definition~\ref{def:kl}) between a Power Spherical distribution $P$ and a uniform distribution on the sphere $Q=\gU(\sS^{d-1})$ is
\begin{equation}
    \mathrm{D_{KL}}[P \| Q] = - \operatorname{H}(P) + \operatorname{H}(Q)
\end{equation}
\begin{proof}
Applying Definition~\ref{def:kl}, $\mathrm{D_{KL}}[P \| Q] =$
\begin{align}
    &= - \operatorname{H}(P) - \E_{p_X}[ \log q(X)] \\
    &= - \operatorname{H}(P) + \operatorname{H}(Q) \;,
\end{align}
where $\operatorname{H}(Q) = \log A_{d-1}$ (from Definition~\ref{def:uniform}).
\end{proof}
\end{theorem}

\clearpage
\onecolumn

\end{document}